\documentclass{elsarticle}
\pdfoutput=1
\usepackage{eurosym}
\usepackage{amsmath}
\usepackage{amssymb}
\usepackage{float}
\usepackage{mathrsfs}
\usepackage[font=small,labelfont=bf]{caption}
\usepackage{float}
\usepackage{varioref}
\usepackage{amssymb,amsmath,mathrsfs}
\usepackage{mathrsfs,enumerate,amssymb,amsmath,color,lineno}
\usepackage{amsfonts}
\usepackage[font=small,labelfont=bf,labelsep=none]{caption}

\newtheorem{theorem}{Theorem}[section]
\newtheorem{algorithm}{Algorithm}
\newtheorem{corollary}{Corollary}[section]
\newtheorem{definition}{Definition}[section]
\newtheorem{example}{Example}[section]
\newtheorem{lemma}{Lemma}[section]
\newtheorem{proposition}{Proposition}[section]
\newtheorem{remark}{Remark}[section]
\newenvironment{proof}[1][Proof]{\noindent\textbf{#1.} }{\hspace{\stretch{1}}$\square$}
\numberwithin{equation}{section}
\begin{document}

\begin{frontmatter}

\title{\huge The solution set of fuzzy relation equations with addition-min composition}
\tnotetext[mytitlenote]{This work is supported by the National Natural Science Foundation of China (No. 12071325).}

\author{Meng Li\footnote{\emph{E-mail address}: 1228205272@qq.com}, Xue-ping Wang\footnote{Corresponding author. xpwang1@hotmail.com; fax:
+86-28-84761393}\\
\emph{School of Mathematical Sciences, Sichuan Normal University,}\\
\emph{Chengdu 610066, Sichuan, People's Republic of China}}

\begin{abstract}
This paper deals with the resolutions of fuzzy relation equations with addition-min composition. When the fuzzy relation equations have a solution, we
first propose an algorithm to find all minimal solutions of the fuzzy relation equations and also supply an algorithm to find all maximal solutions of
the fuzzy relation equations, which will be illustrated, respectively, by numeral examples. Then we prove that every solution of the fuzzy relation
equations is between a minimal solution and a maximal one, so that we describe the solution set of the fuzzy relation equations completely.
\end{abstract}

\begin{keyword}
 Fuzzy relation equation; Addition-min composition; Minimal solution; Maximal solution; Algorithm
\end{keyword}
\end{frontmatter}

\section{Introduction}
The system of fuzzy relation equations with max-min composition was first studied by E. Sanchez \cite{Sanchez76}. The complete solution set of such
finite system is determined by the unique maximal solution and a finite number of minimal solutions \cite{DiNola89}. In recent year, the theory of fuzzy
relation inequalities with addition-min composition has been applied to BitTorrent-like peer-to-peer (P2P) file sharing system widely \cite{Li2012}. In
such file sharing system, let $ A_1, A_2, \cdots , A_n$  be $n$ users who are downloading some file datas simultaneously. Suppose that the $jth$ user
$A_j$ sends the file data with quality level $x_j$ to $A_i$, and the bandwidth between $A_i$ and $A_j$ is $a_{ij}$. Because of the bandwidth limitation,
the network traffic that $A_i$ receives from $A_j$ is actually $a_{ij}\wedge x_j$, and $A_i$ receiving the file data from other users is $a_{i1}\wedge
x_1 + \cdots +a_{i{}i-1}\wedge x_{i-1}+a_{i{i+1}}\wedge x_{i+1}+\cdots+ a_{in}\wedge x_n$. If the quality requirement of download traffic of $A_i$ is at
least $b_i$ ($b_i > 0$), then the data transmission mechanism in the file sharing systems is reduced to the system of fuzzy relation inequations with
addition-min composition as follows:
\begin{equation}\label{eq1}
\left\{\begin{array}{ll}
a_{11}\wedge x_1 + a_{12}\wedge x_{2} + \cdots + a_{1n}\wedge x_n\geq b_1,\\
a_{21}\wedge x_1 + a_{22}\wedge x_{2} + \cdots + a_{2n}\wedge x_n\geq b_2,\\
\cdots\\
a_{m1}\wedge x_1 + a_{m2}\wedge x_{2} + \cdots + a_{mn}\wedge x_n\geq b_m,
\end{array}\right.
\end{equation}
where $a_{ij},x_{j} \in [0,1]$, $b_i > 0$ $(i = 1, 2, \ldots, m; j = 1, 2, \ldots, n)$, $a_{ij} \wedge x_{j} = \min \{a_{ij}, x_j\}$, and the operation
$``+"$ is the ordinary addition. The solution set of the system (\ref{eq1}) and the related optimization problems have been
investigated\cite{Yang2020,Cao,Guo,Guu,Yang2014,SJYang,Yang2018,Yang2017,Li2021}.

In a circumstances of wireless communication, the information is transmitted by the electromagnetic wave. As we know, a high radiation intensity can
ensure a good communication quality, but it will also damage the health of humans. Therefore, in the BitTorrent-like P2P file sharing systems, we are
usually facing that the quality requirement of download traffic of every $A_i$ is exactly required to equal $b_i$ ($b_i > 0$) assigned. Then the file
sharing system (\ref{eq1}) is transformed to fuzzy relation equations with addition-min composition as follows.
\begin{equation}\label{eq2}
\left\{\begin{array}{ll}
a_{11}\wedge x_1 + a_{12}\wedge x_{2} + \cdots + a_{1n}\wedge x_n= b_1,\\
a_{21}\wedge x_1 + a_{22}\wedge x_{2} + \cdots + a_{2n}\wedge x_n= b_2,\\
\cdots\\
a_{m1}\wedge x_1 + a_{m2}\wedge x_{2} + \cdots + a_{mn}\wedge x_n=b_m.
\end{array}\right.
\end{equation}
 In \cite{SJYang}, the system (\ref{eq2}) is called the corresponding equations system of system  (\ref{eq1}), and Yang gave an equivalent condition for
 a solution of the system (\ref{eq2}) to be a minimal one. This article will focuses on the structure of the solution set of system (\ref{eq2}) by
 characterizing all minimal and maximal solutions of system (\ref{eq2}).

 The rest of this article is organized as follows. In Section 2, we present some relevant definitions and basic properties. In Section 3, we first give
 a necessary and sufficient condition for a solution of system (\ref{eq2}) being a minimal one, then we present an algorithm to find all minimal
 solutions. In Section 4, we first show a necessary and sufficient condition for a solution of system (\ref{eq2}) being a maximal solution, and then
 suggest an algorithm to find all maximal solutions. In Section 5, we first prove that every solution of system (\ref{eq2}) is between a minimal
 solution and a maximal one if system (\ref{eq2}) has a solution, then we describe the solution set of system (\ref{eq2}). A conclusion is drawn in
 Section 6.

\section{Preliminaries}

This section presents some basic definitions and results.

Let $I = \{1, 2, \cdots , m\}$ and $J = \{ 1, 2, \cdots, n \}$ be two index sets. For $x^1 = (x_1^1, x_2^1, \cdots, x_n^1)$, $x^2 = (x_1^2, x_2^2,
\cdots, x_n^2) \in [0,1]^{n}$, define ${x^1} \le {x^2}$ if and only if ${x_j}^1 \le {x_j}^2$ for arbitrary $j \in J$, and define ${x^1} < {x^2}$ if and
only if ${x_j}^1 \le {x_j}^2$ and there is a ${j_0} \in J$ such that $x_{j_0}^1 < x_{j_0}^2$. Then system (\ref{eq2}) can be tersely described as
follows:
$$\sum_{j\in J}a_{ij}\wedge x_{j}= b_i, \forall i\in I$$
or
$$A \odot {x^T} = {b^T},$$
where $A=(a_{ij})_{m \times n}$, $x=(x_1, x_2, \cdots, x_n)$, $b=(b_1, b_2, \cdots, b_m)$ and
$$(a_{i1}, a_{i2}, \cdots , a_{in}) \odot ( x_1, x_2, \cdots, x_n)^T = a_{i1} \wedge x_1 + a_{i2} \wedge x_2 + \cdots + a_{in} \wedge x_n.$$ System
(\ref{eq2}) is called solvable if there is an $x \in [0,1]^{n}$ such that $A \odot {x^T} = {b^T}$. We denote the set of all the solutions of system
(\ref{eq2}) by $\mathscr{X} = \{x \in [0,1]^{n} | A \odot x^T = b^T\}$, and $A\setminus B=\{x\in A \mid x\notin B\}$ where $A$ and $B$ are two sets.

\begin{definition}[\cite{DiNola89}]
Let $\mathscr{X}\neq \emptyset$. Then $\hat x \in \mathscr{X}$ is called a maximal solution of system (\ref{eq2}) if and only if $x \geq \hat{x}$ for
any $x \in \mathscr{X}$, implies $x = \hat{x}$, and $\check{x}\in\mathscr{X}$ is called a minimal solution of system (\ref{eq2}) if and only if $x \le
\check{x}$ for any $x \in \mathscr{X}$, implies $x = \check{x}$.
\end{definition}

Denote the set of all minimal solutions of system (\ref{eq2}) by $\check{\mathscr{X}}$ and the set of all maximal solutions of system (\ref{eq2}) by
$\hat{\mathscr{X}}$. Then the following two propositions are obvious.
\begin{proposition}\label{pr2.2}
Let $x \in \mathscr{X}$. Then we have
\begin{enumerate}
\item [(1)] $x > 0$.
\item [(2)] For any $i \in I$, $j \in J$, $x_j \ge b_i - \sum\limits_{k \in J \setminus \{j\}} {a_{ik} \wedge x_k} \ge  b_i - \sum\limits_{k \in J
    \setminus\{ j \}} {a_{ik}} $.
\item [(3)] For any $i \in I$, $j \in J$, $a_{ij} \ge b_i - \sum\limits_{k \in J\setminus\{ j \}} {a_{ik} \wedge x_k} \ge b_i - \sum\limits_{k \in J
    \setminus\{ j\}} a_{ik}$.
\end{enumerate}
\end{proposition}
\begin{proposition}\label{prop3}
Let $x^1, x^2\in \mathscr{X}$. If $x^1\leq x\leq x^2$, then $x\in \mathscr{X}$.
\end{proposition}
\section{Minimal solutions of system (\ref{eq2})}
In this section, we first prove a necessary and sufficient condition for a solution of system (\ref{eq2}) being a minimal one if
$\mathscr{X}\neq\emptyset$, then we propose an algorithm to determine $\check{\mathscr{X}}$.

The following lemma is an equivalent description of Lemma 14 in \cite{SJYang}.
\begin{lemma}\label{le3.1}
Let $x = (x_1, x_2, \cdots, x_n)\in \mathscr{X}$. Then $x\in \check{\mathscr{X}}$ if and only if for any $j \in J$, there is an $i \in I$ such that $x_j
\le a_{ij}$.
\end{lemma}

Let $\hat{\alpha} = (\hat{\alpha}_1, \hat{\alpha}_2, \cdots, \hat{\alpha}_n)$ with
\begin{equation}\label{eq3}
 \hat{\alpha}_j= \max \{a_{ij}|i\in I\}.
\end{equation}
Then from Lemma \ref{le3.1}, we have the following remark.
\begin{remark}\label{re3.1}
Let $x = (x_1, x_2, \cdots, x_n)\in \mathscr{X}$. Then $x\in \check{\mathscr{X}}$ if and only if $x\leq \hat{\alpha}$.
\end{remark}

Let $\check{\alpha}= (\check{\alpha}_1, \check{\alpha}_2, \cdots, \check{\alpha}_n)$ with
\begin{equation}\label{eq4}
 \check{\alpha}_j = \max \{0, b_i-\sum \limits_{k \in J \setminus \{ j \}}a_{ik}|i\in I\}.
\end{equation}
Then we have the following proposition.
\begin{proposition}\label{pr3.1}
For any $x = (x_1, x_2, \cdots, x_n)\in \mathscr{X}$, we have $x \geq \check{\alpha}$.
\end{proposition}
\begin{proof}
Let $x = (x_1, x_2, \cdots, x_n)\in \mathscr{X}$. According to Proposition \ref{pr2.2}, we have $x_j\geq0$ and $x_j \ge b_i - \sum\limits_{k \in
J\setminus\{ j \}} {a_{ik}}$ for any $i \in I$, $j \in J$, hence, $x_j\geq \check{\alpha}_j$.
\end{proof}

 Remark \ref{re3.1} and Proposition \ref{pr3.1} imply the following two remarks.
 \begin{remark}\label{re3.2}
For any $x = (x_1, x_2, \cdots, x_n)\in \check{\mathscr{X}}$, we have $\check{\alpha}\leq x\leq \hat{\alpha}$.
\end{remark}
\begin{remark}\label{re3.3}
If $\check{\alpha}= (\check{\alpha}_1, \check{\alpha}_2, \cdots, \check{\alpha}_n)\in \mathscr{X}$. Then $\check{\alpha}$ is the unique minimal solution
of system (\ref{eq2}).
\end{remark}
\begin{definition}\label{de3.1}
For any $j\in J$, denote $|\{a_{ij}|\check {\alpha}_j< a_{ij}, i\in I\}|=t_j$ and $Q_{j} = \{q_{0j}, q_{1j}, q_{2j}, \cdots, q_{t_jj}\}$ with $\check
{\alpha}_j= q_{0j}<  q_{1j} <  q_{2j}< \cdots < q_{t_jj} = \hat{\alpha}_j$, where $q_{kj}\in \{a_{ij}| i\in I\} (k= 1, 2, \cdots, t_j)$.
\end{definition}

Notice that if $t_j=0$ in Definition \ref{de3.1} then $\check {\alpha}_j = \hat{\alpha}_j$.

Denote
 \begin{equation}\label{eqj}
 J^*=\{j\in J|\check {\alpha}_j=\hat{\alpha}_j\},
  \end{equation}
and $K = \{(k_1, k_2, \cdots, k_n)|k_j\in K_j\mbox{ for any }j\in J\}$ with
\begin{equation}\label{eq5}
   K_j= \left\{\begin{array}{ll}
\{0\}, j\in J^*,\\ \\
\{k|\check {\alpha}_j< q_{kj}\leq \hat{\alpha}_j \mbox{ and } q_{kj}\in Q_j\}, j\in J\setminus J^*.
\end{array}\right.
\end{equation}
Then from Lemma \ref{le3.1}, we have the following corollary.
\begin{corollary}\label{Co3.1}
 Let $x = (x_1, x_2, \cdots, x_n)\in \mathscr{X}$. Then $x \in\check{\mathscr{X}}$ if and only if there exists an $(k_1, k_2, \cdots, k_n)\in K$ such
 that $x_j\leq  q_{k_jj}$ for any $j\in J$.
\end{corollary}
\begin{proof}
Let $x = (x_1, x_2, \cdots, x_n)\in\check{\mathscr{X}}$. By Lemma \ref{le3.1} and Remark \ref{re3.2}, for any $j\in J$ there is an $i\in I$ such that
$\check {\alpha}_j\leq  x_j\leq a_{ij}\leq\hat{\alpha}_j$, which means that $\check {\alpha}_j=x_j=a_{ij}=\hat{\alpha}_j$ for any $j\in J^*$ and $\check
{\alpha}_j\leq x_j\leq a_{ij}\leq\hat{\alpha}_j$ for any $j\in J\setminus J^*$. Then $K_j\neq \emptyset$. Therefore, there exists a $k=(k_1, k_2,
\cdots, k_n)\in K$ such that $x_j\leq  q_{{k_j}j}$ for any $j\in J$.

Conversely, suppose that there is a $k=(k_1, k_2, \cdots, k_n)\in K$ such that $x_j\leq  q_{k_jj}$ for any $j\in J$. It is obvious that for any $j\in
J$, there is an $i\in I$ such that $x_j\leq q_{k_jj}=a_{ij}$ and $x\in \mathscr{X}$, then $x\in \check{\mathscr{X}}$ by Lemma \ref{le3.1}.
\end{proof}

Furthermore, we have the following theorem.
\begin{theorem}\label{th3.1}
Let $x = (x_1, x_2, \cdots, x_n)$. Then $x\in \check{\mathscr{X}}$ if and only if there is an $(k_1, k_2, \cdots, k_n)\in K$ such that $x$ satisfies the
following linear equations:
\begin{equation}\label{eq6}\left\{\begin{array}{ll}
x_{j} = q_{0j}, j\in J^*,\\ \\
q_{({k_j-1})j} \leq x_{j} \leq q_{k_jj}, j\in J\setminus J^*,\\ \\
\sum \limits_{j\in J^*}a_{ij}\wedge q_{0j}+\sum \limits_{j\in J\setminus J^*}[\delta_{ij}\cdot x_j + (1 - \delta_{ij}) \cdot a_{ij}] = b_i, i\in I
\end{array}\right.
\end{equation}
where the operation $``\cdot"$ represents the ordinary multiplication and $$\delta_{ij} = \left\{ \begin{array}{l}
1, q_{k_jj}\leq a_{ij},\\
0, a_{ij}\leq q_{({k_j-1})j}
\end{array} \right.$$
for all $i\in I$, $j\in J\setminus J^*$.
\end{theorem}
 \begin{proof} Suppose that $x\in \check{\mathscr{X}}$. By Corollary \ref{Co3.1} and Formula (\ref{eq5}), there is an $(k_1, k_2, \cdots, k_n)\in K$
 such that $q_{({k_j-1})j} \leq x_{j} \leq q_{k_jj}$ for any $j\in  J\setminus J^*$ and $x_{j}=q_{0j}$ for any $j\in J^*$. Then for any $i\in I$, we
 have
 \begin{eqnarray*}
 b_i &=& \sum \limits_{j\in J} a_{ij}\wedge x_j
 \\&=&\sum \limits_{j\in J^*}a_{ij}\wedge x_{j}+\sum \limits_{j\in J\setminus J^*}a_{ij}\wedge x_{j}
 \\&=&\sum \limits_{j\in J^*}a_{ij}\wedge q_{0j}+\sum \limits_{j\in J\setminus J^*, q_{k_jj}\leq a_{ij}}a_{ij}\wedge x_j+ \sum \limits_{j\in J\setminus
 J^*, a_{ij}\leq q_{({k_j-1})j}}a_{ij}\wedge x_j
 \\&=& \sum \limits_{j\in J^*}a_{ij}\wedge q_{0j}+\sum \limits_{j\in J\setminus J^*, q_{k_jj}\leq a_{ij}}x_j+ \sum \limits_{j\in J\setminus J^*,
 a_{ij}\leq q_{({k_j-1})j}}a_{ij}
 \\ &=&\sum \limits_{j\in J^*}a_{ij}\wedge q_{0j}+\sum \limits_{j\in J\setminus J^*}\delta_{ij}\cdot x_j + \sum \limits_{j\in J\setminus J^*}(1 -
 \delta_{ij})\cdot a_{ij}
 \\&=&\sum \limits_{j\in J^*}a_{ij}\wedge q_{0j}+\sum \limits_{j\in J\setminus J^*}[\delta_{ij}\cdot x_j + (1 - \delta_{ij}) \cdot a_{ij}],
 \end{eqnarray*}
where the operation $``\cdot"$ represents the ordinary multiplication and $$\delta_{ij} = \left\{ \begin{array}{l}
1, q_{k_jj}\leq a_{ij},\\
0, a_{ij}\leq q_{({k_j-1})j}
\end{array} \right.$$
 for all $i\in I$, $j\in J\setminus J^*$.

 Now, suppose that there is an $(k_1, k_2, \cdots, k_n)\in K$ such that $x$ satisfies the linear equations (\ref{eq6}). Then for any $i\in I$, we have
  \begin{eqnarray*}
  b_i&=&\sum \limits_{j\in J^*}a_{ij}\wedge q_{0j}+\sum \limits_{j\in J\setminus J^*}[\delta_{ij}\cdot x_j + (1 - \delta_{ij}) \cdot a_{ij}]
  \\&=&\sum \limits_{j\in J^*}a_{ij}\wedge x_j+\sum \limits_{j\in J\setminus J^*}\delta_{ij}\cdot x_j + \sum \limits_{j\in J\setminus J^*}(1 -
  \delta_{ij})\cdot a_{ij}
  \\&=& \sum \limits_{j\in J^*}a_{ij}\wedge x_j+\sum \limits_{j\in J\setminus J^*, q_{k_jj}\leq a_{ij}}x_j+ \sum \limits_{j\in J\setminus J^*,
  a_{ij}\leq q_{({k_j-1})j}}a_{ij}
   \\&=&\sum \limits_{j\in J^*}a_{ij}\wedge x_j+\sum \limits_{j\in J\setminus J^*, q_{k_jj}\leq a_{ij}}a_{ij}\wedge x_j+ \sum \limits_{j\in J\setminus
   J^*, a_{ij}\leq q_{({k_j-1})j}}a_{ij}\wedge x_j
 \\&=&\sum \limits_{j\in J^*}a_{ij}\wedge x_j+\sum \limits_{j\in J\setminus J^*, q_{k_jj}\leq a_{ij}}a_{ij}\wedge x_j+ \sum \limits_{j\in J\setminus
 J^*, a_{ij}\leq q_{({k_j-1})j}}a_{ij}\wedge x_j
 \\&=&\sum \limits_{j\in J^*}a_{ij}\wedge x_{j}+\sum \limits_{j\in J\setminus J^*}a_{ij}\wedge x_{j}
  \\&=&\sum \limits_{j\in J} a_{ij}\wedge x_j.
 \end{eqnarray*}
 Hence $x\in \mathscr{X}$. By Corollary \ref{Co3.1}, we get $x\in \check{\mathscr{X}}$.
\end{proof}

Moreover, the proof of Theorem \ref{th3.1} implies the following theorem.
\begin{theorem}\label{th3.2}
For any $(k_1, k_2, \cdots, k_n)\in K$, if $x= (x_1, x_2, \cdots, x_n)$ satisfies the equations (\ref{eq6}), then $x\in \check{\mathscr{X}}$.
\end{theorem}

Let the solution set of linear equations (\ref{eq6}) corresponding to $k=(k_1, k_2, \cdots, k_n)\in K$ be represented by $\breve{\mathscr{X}}(k)$. Based
on Theorems \ref{th3.1} and \ref{th3.2}, we obtain an algorithm to find all minimal solutions of system (\ref{eq2}) as follows.
\begin{algorithm}\label{alt1}Input $A=(a_{ij})_{m \times n}$ and $b=(b_1, b_2, \cdots, b_m)$. Output $\check{\mathscr{X}}$.\\
Step 1. Compute $\check{\alpha}= (\check{\alpha}_1, \check{\alpha}_2, \cdots, \check{\alpha}_n)$ defined by (\ref{eq4}). \\
Step 2. If $\check{\alpha}\notin \mathscr{X}$, then compute $\hat{\alpha}_j$ defined by (\ref{eq3}) for any $j\in J$. Otherwise, output
$\check{\mathscr{X}}=\{\check{\alpha}\}$, and go to Step 7.\\
Step 3. Compute $J^*$ defined by (\ref{eqj}).\\
Step 4. Compute $K_j$ defined by (\ref{eq5}) for any $j\in J$,  and $K = \{(k_1, k_2, \cdots, k_n) |k_j\in K_j\}$.\\
Step 5. For any $k=(k_1, k_2, \cdots, k_n)\in K$, construct the corresponding linear equations (\ref{eq6}).\\
Step 6. Solve the corresponding linear equations (\ref{eq6}) and obtain $\check{\mathscr{X}}(k)$.\\
Step 7. Output $\check{\mathscr{X}}=\bigcup \limits_{k\in K}\check{\mathscr{X}}(k)$.\\
Step 8. End.
\end{algorithm}

The following example will be given to illustrate Algorithm \ref{alt1}.
\begin{example}\label{ex1}
\emph{Consider the following fuzzy relation equations:}
$$\left\{ \begin{array}{l}
0.4\wedge x_1 + 0.6\wedge x_2  + 0.5\wedge x_3= 1.4,\\
0.7\wedge x_1 + 0.5\wedge x_2  + 0.8\wedge x_3 = 1.5.\\
\end{array} \right.$$
\end{example}
Step 1. Compute $\check{\alpha}_1 = 0.3$, $\check{\alpha}_2 = 0.5$, $\check{\alpha}_3 = 0.4$.\\
Step 2. It is clear that $\check{\alpha}= (0.3, 0.5, 0.4 )\notin \mathscr{X}$, then compute $\hat{\alpha}_1= 0.7$, $\hat{\alpha}_2= 0.6$ and
$\hat{\alpha}_3= 0.8$.\\
Step 3. For any $j\in J$, $\check{\alpha}_j\neq \hat{\alpha}_j$, thus $J^*=\emptyset$.\\
Step 4. Compute $K_1=\{1, 2\}$, $K_2=\{1\}$, $K_3=\{1, 2\}$ and $$K = \{(1, 1, 1), (1, 1, 2), (2, 1, 1), (2, 1, 2)\}.$$\\
Step 5. Construct and solve the following linear equations (\ref{eq6}):

(1) For $k = (1, 1, 1)$, $$\left\{ \begin{array}{l}
0.3\leq x_1 \leq0.4,\\
0.5\leq x_2 \leq0.6,\\
0.4\leq x_3 \leq0.5,\\
x_1 + x_2  + x_3= 1.4,\\
x_1 + 0.5 + x_3 = 1.5.
\end{array} \right.$$
Hence by $$\left\{ \begin{array}{l}x_1 + x_2  + x_3= 1.4,\\
x_1 + 0.5 + x_3 = 1.5,
\end{array} \right.$$ we have
$$\left\{ \begin{array}{l}
x_1 + x_2  + x_3= 1.4,\\
x_1 + x_3 = 1.
\end{array} \right.$$
Thus $x_2=0.4$, a contradiction since $0.5\leq x_2 \leq0.6$. Therefore, $\check{\mathscr{X}}(1, 1, 1) = \emptyset$.

(2) For $k = (1, 1, 2)$, $$\left\{ \begin{array}{l}
0.3\leq x_1 \leq0.4,\\
0.5\leq x_2 \leq0.6,\\
0.5\leq x_3 \leq0.8,\\
x_1 + x_2  + 0.5= 1.4,\\
x_1 + 0.5 + x_3 = 1.5.
\end{array} \right.$$
Hence by $$\left\{ \begin{array}{l}
x_1 + x_2  + 0.5= 1.4,\\
x_1 + 0.5 + x_3 = 1.5,
\end{array} \right.$$ we have $$\left\{ \begin{array}{l}
x_1 + x_2 = 0.9,\\
x_1 + x_3 = 1.
\end{array} \right.$$
Let $x_1 = t$. Then $x_2 = 0.9-t$ and  $x_3 = 1-t$. Since $0.3\leq x_1 \leq0.4$, $0.5\leq x_2 \leq0.6$ and $0.5\leq x_3 \leq0.8$, $t\in [0.3, 0.4]$.
Therefore, $\check{\mathscr{X}}(1, 1, 2) = \{(t, 0.9-t, 1-t)|t\in [0.3, 0.4]\}$.

(3) For $k = (2, 1, 1)$, $$\left\{ \begin{array}{l}
0.4\leq x_1 \leq0.7,\\
0.5\leq x_2 \leq0.6,\\
0.4\leq x_3 \leq0.5,\\
0.4 + x_2  + x_3= 1.4,\\
x_1 + 0.5 + x_3 = 1.5.
\end{array} \right.$$
Hence by $$\left\{ \begin{array}{l}
0.4 + x_2  + x_3= 1.4,\\
x_1 + 0.5 + x_3 = 1.5,
\end{array}, \right.$$ we have $$\left\{ \begin{array}{l}
x_2 + x_3 = 1,\\
x_1 + x_3 = 1.
\end{array} \right.$$
Let $x_3= t$. Then $x_1= 1-t$ and $x_2= 1-t$. Since $0.4\leq x_1 \leq0.7$, $0.5\leq x_2 \leq0.6$ and $0.4\leq x_3 \leq0.5$, $t\in [0.4, 0.5]$.
Therefore, $\check{\mathscr{X}}(2, 1, 1) = \{(1-t, 1-t, t)|t\in [0.4, 0.5]\}$.

(4) For $k= (2, 1, 2)$, $$\left\{ \begin{array}{l}
0.4\leq x_1 \leq0.7,\\
0.5\leq x_2 \leq0.6,\\
0.5\leq x_3 \leq0.8,\\
0.4 + x_2  + 0.5= 1.4,\\
x_1 + 0.5 + x_3 = 1.5.
\end{array} \right.$$
Hence by $$\left\{ \begin{array}{l}
0.4 + x_2  + 0.5= 1.4,\\
x_1 + 0.5 + x_3 = 1.5,
\end{array} \right.$$ we have $$\left\{ \begin{array}{l}
x_2 = 0.5,\\
x_1 + x_3 = 1.
\end{array} \right.$$
Let $x_1=t$. Then $x_3= 1-t$. Since $0.4\leq x_1 \leq0.7$, $0.5\leq x_2 \leq0.6$ and $0.5\leq x_3 \leq0.8$, $t\in [0.4, 0.5]$. Therefore,
$\check{\mathscr{X}}(2, 1, 2) = \{(t, 0.5, 1-t)|t\in [0.4, 0.5]\}$.\\
Step 6. \begin{eqnarray*}\check{\mathscr{X}}&=& \bigcup \limits_{k\in K}\check{\mathscr{X}}(k)\\&=& \{(t, 0.9-t, 1-t)|t\in [0.3, 0.4]\}\cup\{(1-t, 1-t,
t)|t\in [0.4, 0.5]\}\\&\quad&\cup\{(t, 0.5, 1-t)|t\in [0.4, 0.5]\}.\end{eqnarray*}
\section{Maximal solutions of system (\ref{eq2})}
In this section, we first show a necessary and sufficient condition for a solution of system (\ref{eq2}) being a maximal solution, and then suggest an
algorithm to find all elements of $\hat{\mathscr{X}}$.
\begin{remark}\label{rek4}
If $(1, 1, \cdots, 1)\in \mathscr{X}$. Then $(1, 1, \cdots, 1)$ is the unique maximal solution of system (\ref{eq2}).
\end{remark}

Let $x = (x_1, x_2, \cdots, x_n)$ and denote $J(x)=\{j\in J|x_j=1\}$. Then we have:
\begin{theorem}\label{th4.1}
Let $x = (x_1, x_2, \cdots, x_n)\in \mathscr{X}$. Then $x\in \hat{\mathscr{X}}$ if and only if for any $j \in J\setminus J(x)$, there is an $i\in I$
such that $x_j < a_{ij}$.
\end{theorem}
\begin{proof}
Let $x = (x_1, x_2, \cdots, x_n )\in \hat{\mathscr{X}}$. Assume that for any $i\in I$, there is a $j \in J\setminus J(x)$ such that $a_{ij}\leq x_j<1$.
Define $x^1 = (x^1_1, x^1_2, \cdots, x^1_s)$ with $$x_k^1 = \left\{ \begin{array}{l}
1, k = j,\\
x_k, k\ne j.
\end{array} \right.$$
Obviously, $x<x^1$. For any $i\in I$, we have
\begin{eqnarray*}
\sum\limits_{k \in J} a_{ik}\wedge x^1_k &=& a_{ij}\wedge x^1_j + \sum\limits_{k \in {J\setminus\{j\}}} a_{ik}\wedge x_k \\ &=&a_{ij} +
\sum\limits_{k\in {J\setminus\{j\}}} a_{ik}\wedge x_k \\ &=& a_{ij}\wedge x_j + \sum\limits_{k\in {J\setminus\{j\}}} a_{ik}\wedge x_k\\ &=& b_i.
\end{eqnarray*}
Hence, $x^1\in \mathscr{X}$, contrary to $x$ is a maximal solution of system (\ref{eq2}).

Conversely, suppose that there exists a $y\in \mathscr{X}$ such $x<y$. Then there is a $h\in J\setminus J(x)$ such that $x_h <y_h$ since $x_j=1$ for
$j\in J(x)$. Thus $x_h <a_{ih}\wedge y_h$ since $x_h< a_{ih}$. So that
\begin{eqnarray*}
\sum\limits_{k \in J} a_{ik}\wedge y_k &=& a_{ih}\wedge y_h + \sum\limits_{k \in {J\setminus\{h\}}} a_{ik}\wedge y_k \\&>& x_h+ \sum\limits_{k \in
{J\setminus\{h\}}} a_{ik}\wedge y_k\\&\geq&x_h+ \sum\limits_{k \in {J\setminus\{h\}}} a_{ik}\wedge x_k \\&=& a_{ih}\wedge x_h +\sum\limits_{k \in
{J\setminus\{h\}}} a_{ik}\wedge x_k\\&=& b_i.
\end{eqnarray*}
Contrary to $y\in \mathscr{X}$. Therefore, $x$ is a maximal solution of system (\ref{eq2}).
\end{proof}

Let $M=\{(m_1, m_2, \cdots, m_n)|m_j\in M_j \mbox{ for any }j\in J\}$ with
\begin{equation}\label{eq7}
   M_j= (K_j\setminus \{0\})\bigcup \{\infty\}.
\end{equation}


Let $m=(m_1, m_2, \cdots, m_n)\in M$. Denote $J(m)=\{j\in J | m_j=\infty\}$. Then we have the following theorem.
\begin{theorem}\label{th4.2}
Let $x = (x_1, x_2, \cdots, x_n)$. Then $x\in \hat{\mathscr{X}}$ if and only if there is an $m=(m_1, m_2, \cdots, m_n)\in M$ such that $x$ satisfies the
linear equations as below:
\begin{equation}\label{eq9}\left\{\begin{array}{ll}
x_j=1, j\in J(m),\\\\
q_{(m_j-1)j} \leq x_j < q_{m_jj}, j\in J\setminus J(m),\\ \\
\sum \limits_{j\in J\setminus J(m)}\delta_{ij}\cdot x_j + \sum \limits_{j\in J\setminus J(m)}(1 - \delta_{ij})\cdot a_{ij} +\sum \limits_{j\in
J(m)}a_{ij}= b_i,~i\in I,
\end{array}\right.
\end{equation}
where the operation $``\cdot"$ represents the ordinary multiplication and $$\delta_{ij} = \left\{ \begin{array}{l}
1, q_{m_jj} \leq a_{ij},\\\\
0, a_{ij} \leq q_{(m_j-1)j}
\end{array} \right.$$
for all $i\in I$, $j\in J\setminus J(m)$.
\end{theorem}
\begin{proof}
 Suppose that $x\in \hat{\mathscr{X}}$. Then by Theorem \ref{th4.1}, there is an $i\in I$ such that $\check{\alpha}_j\leq x_j<  q_{m_jj}=a_{ij}$ for any
 $j\in J\setminus J(x)$, and by Formula (\ref{eq5}) it is easy to see that $m_j\in K_j\setminus \{0\}$. For any $j\in J$, let $m_j=\infty$ if $j\in
 J(x)$. Thus $J(m)=J(x)$. In this way, we can construct an $m=(m_1, m_2, \cdots, m_n)$ such that $x_j<  q_{m_jj}$ for any $j\in J\setminus J(m)$ and
 $x_j=1$ for any $j\in J(m)$. By Formula (\ref{eq7}), $m=(m_1, m_2, \cdots, m_n)\in M$. Again by Formula (\ref{eq5}), we have $q_{(m_j-1)j} \leq x_j <
 q_{m_jj}$ for any $j\in J\setminus J(m)$. Then for any $i\in I$, \begin{eqnarray*}
 b_i &=& \sum \limits_{j\in J} a_{ij}\wedge x_j \\&=&  \sum \limits_{j\in J\setminus J(m)} a_{ij}\wedge x_j + \sum \limits_{j\in J(m)} a_{ij}\wedge x_j
 \\&=& \sum \limits_{j\in J\setminus J(m), q_{m_jj} \leq a_{ij}}a_{ij}\wedge x_j+ \sum
\limits_{j\in J\setminus J(m), a_{ij}\leq q_{({m_j-1})j}}a_{ij}\wedge x_j + \sum \limits_{j\in J(m)} a_{ij}\wedge x_j  \\&=& \sum \limits_{j\in
J\setminus J(m), q_{m_jj} \leq a_{ij}}x_j+ \sum \limits_{j\in J\setminus J(m), a_{ij}\leq q_{({m_j-1})j}}a_{ij} + \sum \limits_{j\in J(m)} a_{ij} \\
&=&\sum \limits_{j\in J\setminus J(m)}\delta_{ij}\cdot x_j + \sum \limits_{j\in J\setminus J(m)}(1 - \delta_{ij})\cdot a_{ij}+ \sum \limits_{j\in J(m)}
a_{ij},
 \end{eqnarray*}
where the operation $``\cdot"$ represents the ordinary multiplication and $$\delta_{ij} = \left\{ \begin{array}{l}
1, q_{m_jj}\leq a_{ij},\\
0, a_{ij}\leq q_{({m_j-1})j}
\end{array} \right.$$
 for all $i\in I$, $j\in J\setminus J(m)$.

  Conversely, suppose that there exists an $(m_1, m_2, \cdots, m_n)\in M$ such that $x$ satisfies the linear equations (\ref{eq9}). Then for any $i\in
  I$, we have \begin{eqnarray*}
  b_i&=& \sum \limits_{j\in J\setminus J(m)}\delta_{ij}\cdot x_j + \sum \limits_{j\in J\setminus J(m)}(1 - \delta_{ij})\cdot
  a_{ij} +\sum \limits_{j\in J(m)}a_{ij}
  \\&=&\sum \limits_{j\in J\setminus J(m), q_{m_jj} \leq a_{ij}}x_j+ \sum \limits_{j\in J\setminus J(m), a_{ij}\leq q_{(m_j-1)j}}a_{ij}+\sum
  \limits_{j\in J(m)}a_{ij}
  \\&=& \sum \limits_{j\in J\setminus J(m), q_{m_jj} \leq a_{ij}}a_{ij}\wedge x_j+ \sum \limits_{j\in J\setminus J(m), a_{ij}\leq
  q_{(m_j-1)j}}a_{ij}\wedge x_j+\sum \limits_{j\in J(m)}a_{ij}\wedge x_j
  \\&=&\sum \limits_{j\in J\setminus J(m)} a_{ij}\wedge x_j + \sum \limits_{j\in J(m)} a_{ij}\wedge x_j
  \\&=& \sum \limits_{j\in J} a_{ij}\wedge x_j.
 \end{eqnarray*}
 Hence $x\in \mathscr{X}$. For this $x$, it is easy to see  $J(m)=J(x)$. So that there is an $i$ such that $ x_j<q_{m_jj}= a_{ij}$ for any $j\in
 J\setminus J(x)$. Consequently, by Theorem \ref{th4.1}, we have $x\in \hat{\mathscr{X}}$.
\end{proof}

Furthermore, the proof of Theorem \ref{th4.2} implies the following theorem.
\begin{theorem}\label{th4.3}
For any $m=(m_1, m_2, \cdots, m_n)\in M$, if $x= (x_1, x_2, \cdots, x_n)$ satisfies the equations (\ref{eq9}), then $x\in \hat{\mathscr{X}}$.
\end{theorem}

Let the solution set of linear equations (\ref{eq9}) corresponding to $m= (m_1, m_2, \cdots, m_n)\in M$ be represented by $\hat{\mathscr{X}}(m)$. Then
according to Theorems \ref{th4.2} and \ref{th4.3}, we obtain an algorithm to find all maximal solutions of system (\ref{eq2}) as follows.
\begin{algorithm}\label{al4.1}Input $A=(a_{ij})_{m \times n}$ and $b=(b_1, b_2, \cdots, b_m)$. Output $\hat{\mathscr{X}}$.\\
Step 1. If $(1, 1, \cdots, 1)\notin\mathscr{X}$, then compute $\hat{\alpha}_j$ and $\check{\alpha}_j$ defined by (\ref{eq3}) and (\ref{eq4}) for any
$j\in J$, respectively. Otherwise, output $\check{\mathscr{X}}=\{(1, 1, \cdots, 1)\}$, and go to Step 6.\\
Step 2. Compute $M_j$ defined by (\ref{eq7}) for any $j\in J$, and $M=\{(m_1, m_2, \cdots, m_n)|m_j\in M_j\}$.\\
Step 3. For any $m= (m_1, m_2, \cdots, m_n)\in M$, construct the corresponding linear equations (\ref{eq9}).\\
Step 4. Solve the corresponding linear equations (\ref{eq9}) and obtain $\hat{\mathscr{X}}(m)$.\\
Step 5. Output $\hat{\mathscr{X}}=\bigcup \limits_{m\in M}\hat{\mathscr{X}}(m)$.\\
Step 6. End.
\end{algorithm}

 Next, an example will be given to verify the effectiveness of Algorithm \ref{al4.1}.
\begin{example}
\emph{Consider the fuzzy relation equations in Example \ref{ex1} again.}
\end{example}
Step 1. It is obvious that $(1, 1, 1)$ is not a solution of the fuzzy relation equations, then compute $\hat{\alpha}_1= 0.7$, $\hat{\alpha}_2= 0.6$ and
$\hat{\alpha}_3= 0.8$; $\check{\alpha}_1 = 0.3$, $\check{\alpha}_2 = 0.5$ and $\check{\alpha}_3 = 0.4$.\\
Step 2. Compute $M_1=\{1, 2, \infty\}$, $M_2=\{1, \infty\}$, $M_3=\{1, 2, \infty\}$, and
\begin{eqnarray*}
M&=&\{(1, 1, 1), (1, 1, 2), (1, 1,\infty), (1, \infty, 1), (1, \infty, 2), (1, \infty, \infty), (2, 1, 1), (2, 1, 2), (2, 1, \infty),\\&&(2, \infty, 1),
(2, \infty, 2), (2, \infty, \infty), (\infty ,1, 1), (\infty, 1, 2), (\infty, 1, \infty), (\infty, \infty, 1), (\infty, \infty, 2), \\&&(\infty, \infty,
\infty)\}.
\end{eqnarray*}
Step 3. Construct and solve the following linear equations (\ref{eq9}):

(1) For $m = (1, 1, 1)$, $$\left\{ \begin{array}{l}
0.3\leq x_1< 0.4,\\
0.5\leq x_2 <0.6,\\
0.4\leq x_3 <0.5,\\
x_1 + x_2  + x_3= 1.4,\\
x_1 + 0.5 + x_3 = 1.5.
\end{array} \right.$$
Hence by $$\left\{ \begin{array}{l}
x_1 + x_2  + x_3= 1.4,\\
x_1 + 0.5 + x_3 = 1.5,
\end{array} \right.$$
we have $$\left\{ \begin{array}{l}
x_1 + x_2  + x_3= 1.4,\\
x_1 + x_3 = 1.
\end{array} \right.$$
$x_2=0.4$, contrary to $0.5\leq x_2 < 0.6$. Therefore, $\check{\mathscr{X}}(1, 1, 1) = \emptyset$.

(2) For $m = (1, 1, 2)$, $$\left\{ \begin{array}{l}
0.3\leq x_1 <0.4,\\
0.5\leq x_2 <0.6,\\
0.5\leq x_3 <0.8,\\
x_1 + x_2  + 0.5= 1.4,\\
x_1 + 0.5 + x_3 = 1.5.
\end{array} \right.$$
Hence by $$\left\{ \begin{array}{l}
x_1 + x_2  + 0.5= 1.4,\\
x_1 + 0.5 + x_3 = 1.5,
\end{array} \right.$$
 we have $$\left\{ \begin{array}{l}
x_1 + x_2 = 0.9,\\
x_1 + x_3 = 1.
\end{array} \right.$$
Let $x_1 = t$. Then $x_2 = 0.9-t$ and  $x_3 = 1-t$. Since $0.3\leq x_1 <0.4$, $0.5\leq x_2 <0.6$ and $0.5\leq x_3 <0.8$, $t\in (0.3, 0.4)$. Therefore,
$\hat{\mathscr{X}}(1, 1, 2) = \{(t, 0.9-t, 1-t)|t\in (0.3, 0.4)\}$.

(3) For $m = (1, 1, \infty)$, $$\left\{ \begin{array}{l}
0.3\leq x_1< 0.4,\\
0.5\leq x_2 <0.6,\\
x_3 =1,\\
x_1 + x_2  + 0.5= 1.4,\\
x_1 + 0.5 + 0.8 = 1.5.
\end{array} \right.$$
Hence by $$\left\{ \begin{array}{l}
x_1 + x_2  + 0.5= 1.4,\\
x_1 + 0.5 + 0.8 = 1.5,
\end{array} \right.$$
we have $$\left\{ \begin{array}{l}
x_2 = 0.7,\\
x_1 = 0.2,
\end{array} \right.$$
contrary to $0.3\leq x_1< 0.4$ and $0.5\leq x_2 <0.6$. Therefore, $\check{\mathscr{X}}(1, 1, \infty) = \emptyset$.

(4) For $m = (1, \infty, 1)$, $$\left\{ \begin{array}{l}
0.3\leq x_1< 0.4,\\
x_2 = 1,\\
0.4\leq x_3 <0.5,\\
x_1 + 0.6  + x_3= 1.4,\\
x_1 + 0.5 + x_3 = 1.5.
\end{array} \right.$$
Hence by $$\left\{ \begin{array}{l}
x_1 + 0.6  + x_3= 1.4,\\
x_1 + 0.5 + x_3 = 1.5,
\end{array} \right.$$
we have $$\left\{ \begin{array}{l}
x_1 + x_3= 0.8,\\
x_1  + x_3 = 1,
\end{array} \right.$$
a contradiction. Therefore, $\check{\mathscr{X}}(1, \infty, 1) = \emptyset$.

(5) For $m = (1, \infty, 2)$, $$\left\{ \begin{array}{l}
0.3\leq x_1< 0.4,\\
x_2 = 1,\\
0.5\leq x_3 <0.8,\\
x_1 + 0.6  + 0.5= 1.4,\\
x_1 + 0.5 + x_3 = 1.5.
\end{array} \right.$$
Hence by $$\left\{ \begin{array}{l}
x_1 + 0.6  + 0.5= 1.4,\\
x_1 + 0.5 + x_3 = 1.5,
\end{array} \right.$$
we have $$\left\{ \begin{array}{l}
x_1 = 0.3,\\
x_3 = 0.7.
\end{array} \right.$$
Therefore, $\hat{\mathscr{X}}(1, \infty, 2)=\{(0.3, 1, 0.7)\}$.

(6) For $m = (1, \infty, \infty)$, $$\left\{ \begin{array}{l}
0.3\leq x_1< 0.4,\\
x_2 = 1,\\
 x_3 =1,\\
x_1 + 0.6  + 0.5= 1.4,\\
x_1 + 0.5 + 0.8 = 1.5.
\end{array} \right.$$
Hence by $$\left\{ \begin{array}{l}
x_1 + 0.6  + 0.5= 1.4,\\
x_1 + 0.5 + 0.8 = 1.5,
\end{array} \right.$$
we have $$\left\{ \begin{array}{l}
x_1 = 0.3,\\
x_1  = 0.2,
\end{array} \right.$$
a contradiction. Therefore, $\check{\mathscr{X}}(1, \infty, \infty) = \emptyset$.

(7) For $m = (2, 1, 1)$, $$\left\{ \begin{array}{l}
0.4\leq x_1 <0.7,\\
0.5\leq x_2 <0.6,\\
0.4\leq x_3 <0.5,\\
0.4 + x_2  + x_3= 1.4,\\
x_1 + 0.5 + x_3 = 1.5.
\end{array} \right.$$
Hence by $$\left\{ \begin{array}{l}
0.4 + x_2  + x_3= 1.4,\\
x_1 + 0.5 + x_3 = 1.5,
\end{array} \right.$$
 we have $$\left\{ \begin{array}{l}
x_2 + x_3 = 1,\\
x_1 + x_3 = 1.
\end{array} \right.$$
Let $x_3= t$. Then $x_1= 1-t$ and $x_2= 1-t$. Since $0.4\leq x_1 <0.7$, $0.5\leq x_2 <0.6$ and $0.4\leq x_3 <0.5$, $t\in (0.4, 0.5)$. Therefore,
$\hat{\mathscr{X}}(2, 1, 1) = \{(1-t, 1-t, t)|t\in (0.4, 0.5)\}$.

(8) For $m= (2, 1, 2)$, $$\left\{ \begin{array}{l}
0.4\leq x_1 <0.7,\\
0.5\leq x_2 <0.6,\\
0.5\leq x_3 <0.8,\\
0.4 + x_2  + 0.5= 1.4,\\
x_1 + 0.5 + x_3 = 1.5.
\end{array} \right.$$
Hence by $$\left\{ \begin{array}{l}
0.4 + x_2  + 0.5= 1.4,\\
x_1 + 0.5 + x_3 = 1.5,
\end{array} \right.$$
 we have $$\left\{ \begin{array}{l}
x_2 = 0.5,\\
x_1 + x_3 = 1.
\end{array} \right.$$
Let $x_1=t$. Then $x_3= 1-t$. Since $0.4\leq x_1 <0.7$, $0.5\leq x_2 <0.6$ and $0.5\leq x_3 <0.8$, $t\in [0.4, 0.5]$. Therefore, $\check{\mathscr{X}}(2,
1, 2) = \{(t, 0.5, 1-t)|t\in [0.4, 0.5]\}$.

(9) For $m = (2, 1, \infty)$, $$\left\{ \begin{array}{l}
0.4\leq x_1 <0.7,\\
0.5\leq x_2 <0.6,\\
 x_3=1,\\
0.4 + x_2  + 0.5= 1.4,\\
x_1 + 0.5 + 0.8 = 1.5.
\end{array} \right.$$
Hence by $$\left\{ \begin{array}{l}
0.4 + x_2  + 0.5= 1.4,\\
x_1 + 0.5 + 0.8 = 1.5,
\end{array} \right.$$
we have $$\left\{ \begin{array}{l}
x_2= 0.5,\\
x_1 = 0.2,
\end{array} \right.$$
contrary to $0.4\leq x_1 <0.7$. Therefore, $\check{\mathscr{X}}(2, 1, \infty) = \emptyset$.

(10) For $m = (2, \infty, 1)$, $$\left\{ \begin{array}{l}
0.4\leq x_1 <0.7,\\
x_2 =1,\\
0.4\leq x_3 <0.5,\\
0.4 + 0.6  + x_3= 1.4,\\
x_1 + 0.5 + x_3 = 1.5.
\end{array} \right.$$
Hence by $$\left\{ \begin{array}{l}
0.4 + 0.6  + x_3= 1.4,\\
x_1 + 0.5 + x_3 = 1.5,
\end{array} \right.$$
 we have $$\left\{ \begin{array}{l}
x_1 = 0.6,\\
x_3 = 0.4.
\end{array} \right.$$
Therefore, $\hat{\mathscr{X}}(2, \infty, 1)=\{(0.6, 1, 0.4)\}$.

(11) For $m = (2, \infty, 2)$, $$\left\{ \begin{array}{l}
0.4\leq x_1 <0.7,\\
x_2 =1,\\
0.5\leq x_3 <0.8,\\
0.4 + 0.6  + 0.5= 1.4,\\
x_1 + 0.5 + x_3 = 1.5.
\end{array} \right.$$

0.4 + 0.6  + 0.5= 1.4 is impossible. Therefore, $\check{\mathscr{X}}(2, \infty, 2) = \emptyset$.

(12) For $m = (2, \infty, \infty)$, $$\left\{ \begin{array}{l}
0.4\leq x_1 <0.7,\\
x_2 =1,\\
x_3 =1,\\
0.4 + 0.6  + 0.5= 1.4,\\
x_1 + 0.5 + 0.8 = 1.5.
\end{array} \right.$$

0.4 + 0.6  + 0.5= 1.4 is impossible. Therefore, $\check{\mathscr{X}}(2, \infty, \infty) = \emptyset$.

(13) For $m = (\infty, 1, 1)$, $$\left\{ \begin{array}{l}
x_1=1,\\
0.5\leq x_2 <0.6,\\
0.4\leq x_3 <0.5,\\
0.4 + x_2  + x_3= 1.4,\\
0.7 + 0.5 + x_3 = 1.5.
\end{array} \right.$$
Hence by $$\left\{ \begin{array}{l}
0.4 + x_2  + x_3= 1.4,\\
0.7 + 0.5 + x_3 = 1.5,
\end{array} \right.$$
we have $$\left\{ \begin{array}{l}
x_2 = 0.7,\\
x_3 = 0.3,
\end{array} \right.$$
contrary to $0.5\leq x_2 <0.6$ and $0.4\leq x_3 <0.5$. Therefore, $\check{\mathscr{X}}(\infty, 1, 1) = \emptyset$.

(14) For $m = (\infty, 1, 2)$, $$\left\{ \begin{array}{l}
x_1=1,\\
0.5\leq x_2 <0.6,\\
0.5\leq x_3 <0.8,\\
0.4 + x_2  + 0.5= 1.4,\\
0.7 + 0.5 + x_3 = 1.5.
\end{array} \right.$$
Hence by $$\left\{ \begin{array}{l}
0.4 + x_2  + 0.5= 1.4,\\
0.7 + 0.5 + x_3 = 1.5,
\end{array} \right.$$
we have $$\left\{ \begin{array}{l}
x_2 = 0.5,\\
x_3 = 0.3,
\end{array} \right.$$
contrary to $0.5\leq x_3 <0.8$, therefore, $\check{\mathscr{X}}(\infty, 1, 2) = \emptyset$.

(15) For $m = (\infty, 1, \infty)$, $$\left\{ \begin{array}{l}
x_1=1,\\
0.5\leq x_2 <0.6,\\
x_3 =1,\\
0.4 + x_2  + 0.5= 1.4,\\
0.7 + 0.5 + 0.8 = 1.5.
\end{array} \right.$$

0.7 + 0.5 + 0.8 = 1.5 is impossible. Therefore, $\check{\mathscr{X}}(\infty, 1, \infty) = \emptyset$.

(16) For $m = (\infty, \infty, 1)$, $$ \left\{ \begin{array}{l}
x_1=1,\\
x_2 =1,\\
0.4\leq x_3 <0.5,\\
0.4 + 0.6  + x_3= 1.4,\\
0.7 + 0.5 + x_3 = 1.5.
\end{array} \right.$$
Hence by $$ \left\{ \begin{array}{l}
0.4 + 0.6  + x_3= 1.4,\\
0.7 + 0.5 + x_3 = 1.5,
\end{array} \right.$$
we have $$ \left\{ \begin{array}{l}
 x_3= 0.4,\\
x_3 = 0.3,
\end{array} \right.$$
a contradiction. Therefore, $\check{\mathscr{X}}(\infty, \infty, 1) = \emptyset$.

(17) For $m = (\infty, \infty, 2)$, $$\left\{ \begin{array}{l}
x_1=1,\\
x_2 =1,\\
0.5\leq x_3 <0.8,\\
0.4 + 0.6  + 0.5= 1.4,\\
0.7 + 0.5 + x_3 = 1.5.
\end{array} \right.$$

0.4 + 0.6  + 0.5= 1.4 is impossible. Therefore, $\check{\mathscr{X}}(\infty, \infty, 2) = \emptyset$.

(18) For $m = (\infty, \infty, \infty)$, $$\left\{ \begin{array}{l}
x_1=1,\\
x_2 =1,\\
x_3 =1,\\
0.4 + 0.6  + 0.5= 1.4,\\
0.7 + 0.5 + 0.8 = 1.5.
\end{array} \right.$$

Both 0.4 + 0.6  + 0.5= 1.4 and 0.7 + 0.5 + 0.8 = 1.5 are impossible. Therefore, $\check{\mathscr{X}}(\infty, \infty, \infty) = \emptyset$.\\
Step 4.
\begin{eqnarray*}
\hat{\mathscr{X}} &= &\bigcup \limits_{m\in M}\hat{\mathscr{X}}(m)\\&=&\{(t, 0.9-t, 1-t)|t\in (0.3, 0.4)\}
\cup\{(0.3, 1, 0.7)\}\cup \{(1-t, 1-t, t)|t\in (0.4, 0.5)\}\\&\quad&\cup\{(t, 0.5, 1-t)|t\in [0.4, 0.5]\}\cup \{(0.6, 1, 0.4)\}.
\end{eqnarray*}

\section{The solution set of system (\ref{eq2})}
In this section, we first prove that every solution of system (\ref{eq2}) is between a minimal solution and a maximal one if $\mathscr{X}\neq\emptyset$,
then we completely describe the solution set $\mathscr{X}$.
\begin{theorem}\label{th4.4}
For every $x\in \mathscr{X}$, there is an $\check{x}\in \check{\mathscr{X}}$ and an $\hat{x}\in \hat{\mathscr{X}}$ such that $\check{x}\leq x\leq
\hat{x}$.
\end{theorem}
\begin{proof} Let $x= (x_1, x_2, \cdots, x_n)\in \mathscr{X}$. If $x\in \check{\mathscr{X}}$, then obviously, there is an $\check{x}\in
\check{\mathscr{X}}$ such that $\check{x}\leq x$.

If $x\notin \check{\mathscr{X}}$, then according to Lemma \ref{le3.1}, for any $ i \in I$ there is a $j \in J$ such that $x_j > a_{ij}$. Thus
$J_1=\{j\in J|x_j>\max\{a_{ij}|i\in I\}\}\neq\emptyset$. Define $x^1 = (x^1_1, x^1_2, \cdots, x^1_n)$ with $$x_j^1 = \left\{ \begin{array}{l}
\max\{a_{ij}|i\in I\}, j\in J_1,\\
x_j, j\in J\setminus J_1.
\end{array} \right.$$
Obviously, $x^1<x$. For any $i\in I$, we have
\begin{eqnarray*}
\sum\limits_{j \in J} a_{ij}\wedge x^1_j &=& \sum\limits_{j\in J_1}a_{ij}\wedge x^1_j + \sum\limits_{j\in J\setminus J_1} a_{ij}\wedge x_j \\
&=&\sum\limits_{j\in J_1}a_{ij} + \sum\limits_{j\in J\setminus J_1} a_{ij}\wedge x_j \\ &=& \sum\limits_{j\in J_1}a_{ij}\wedge x_j + \sum\limits_{j\in
{J\setminus J_1}} a_{ij}\wedge x_j\\ &=& b_i.
\end{eqnarray*}
Hence, $x^1\in \mathscr{X}$. From the structurer of $x^1$ one can verify that for any $j \in J$ there is an $i\in I$ such that $x^1_j \leq a_{ij}$. By
Lemma \ref{le3.1}, $x^1\in \check{\mathscr{X}}$ and $x^1<x$.

Consider the $x= (x_1, x_2, \cdots, x_n)\in \mathscr{X}$ again.  If $x\in \hat{\mathscr{X}}$, then obviously, there is an $\hat{x}\in \hat{\mathscr{X}}$
such that $x\leq \hat{x}$.

If $x\notin\hat{\mathscr{X}}$, then by Remark \ref{rek4} $J(x)\neq J$, and by Theorem \ref{th4.1}, for any $ i \in I$ there is a $j \in J\setminus J(x)$
such that $x_j \geq a_{ij}$. Thus $J_2=\{j\in J\setminus J(x)|x_j\geq\max\{a_{ij}|i\in I\}\}\neq\emptyset$. Define $x^2 = (x^2_1, x^2_2, \cdots, x^2_n)$
with $$x_j^2 = \left\{ \begin{array}{l}
1, j\in J_2,\\
x_j, j\in J\setminus J_2.
\end{array} \right.$$
Obviously, $x<x^2$. For any $i\in I$, we have
\begin{eqnarray*}
\sum\limits_{j\in J} a_{ij}\wedge x^2_j &=& \sum\limits_{j\in J_2}a_{ij}\wedge x^2_j + \sum\limits_{j \in {J\setminus J_2}} a_{ij}\wedge x_j \\
&=&\sum\limits_{j\in J_2}a_{ij} + \sum\limits_{j\in {J\setminus J_2}} a_{ikj}\wedge x_j \\ &=& \sum\limits_{j\in J_2}a_{ij}\wedge x_j +
\sum\limits_{j\in {J\setminus J_2}} a_{ij}\wedge x_j\\ &=& b_i.
\end{eqnarray*}
Hence, $x^2\in \mathscr{X}$. One can check that $J(x^2)=J_2\cup (J(x)\cap (J\setminus J_2))$, and for any $j\in J\setminus J(x^2)$ there is an $ i \in
I$ such that $x^2_j < a_{ij}$. From Theorem \ref{th4.1}, $x^2\in \hat{\mathscr{X}}$ and $x<x^2$.

Consequently, there is an $\check{x}\in \check{\mathscr{X}}$ and an $\hat{x}\in \hat{\mathscr{X}}$ such that $\check{x}\leq x\leq \hat{x}$.
\end{proof}

From Proposition \ref{prop3} and Theorem \ref{th4.4}, we have:
 \begin{theorem}\label{th4.5} If $\mathscr{X}\neq\emptyset$, then
  $$\mathscr{X}=\bigcup\limits_{\check{x}\in \check{\mathscr{X}}, \hat{x}\in \hat{\mathscr{X}}}\{x\in [0,1]^n \mid \check{x}\le x \le\hat{x}\}.$$
\end{theorem}

The solution set of the fuzzy relation equations in Example \ref{ex1} is
\begin{eqnarray*}
\mathscr{X} &=&\bigcup\limits_{\check{x}\in \check{\mathscr{X}}, \hat{x}\in \hat{\mathscr{X}}}\{x\in [0,1]^n \mid \check{x}\le x \le\hat{x}\}\\&=&\{(t,
0.9-t, 1-t)|t\in (0.3, 0.4)\}\cup\{(1-t, 1-t, t)|t\in (0.4, 0.5)\}\\&\quad&\cup\{(t, 0.5, 1-t)|t\in [0.4, 0.5]\}\cup \{(0.6, 0.6, 0.4)\leq x \leq(0.6,
1, 0.4)\}\\&\quad&\cup\{(0.3, 0.6, 0.7)\leq x \leq(0.3, 1, 0.7)\}.
\end{eqnarray*}
\section{Conclusions}
This paper first gave two algorithms to find all minimal and maximal solutions of system (\ref{eq2}) if $\mathscr{X}\neq \emptyset$, respectively, and
then showed that the solution set of system (\ref{eq2}) is completely determined by all minimal and maximal solutions. Algorithm \ref{alt1} can
determine a lot of minimal solutions of system (\ref{eq1}) if $\mathscr{X}\neq \emptyset$ since every minimal solution of system (\ref{eq2}) is also a
minimal one of system (\ref{eq1}). In the future, all minimal solutions found by Algorithm \ref{alt1} will be used for studying the optimization
problems with system (\ref{eq2}) as constraints. Also, it is worth pointing out that we do not know the solvable condition of system (\ref{eq2}).

\end{document}